\documentclass[conference,letterpaper]{IEEEtran}
\usepackage[letterpaper, left=0.75in, right=0.75in, bottom=0.75in, top=0.75in]{geometry}

\IEEEoverridecommandlockouts

\usepackage{array}
\usepackage{hyperref}
\usepackage{cite}
\usepackage{amsmath,amssymb,amsfonts}
\usepackage{algorithm}
\usepackage{algorithmic}
\usepackage{graphicx}
\usepackage{adjustbox}
\usepackage{textcomp}
\usepackage{xcolor}
\usepackage{booktabs}
\usepackage{subcaption}
\usepackage{caption}
\usepackage{multirow}
\usepackage{enumitem}
\usepackage{amsthm}
\newtheorem{proposition}{Proposition}
\usepackage{tikz}
\usetikzlibrary{positioning,arrows.meta,fit}
\begin{document}

\title{Cost-Optimal Foundation Model Deployment Portfolio for Transportation Management}


\author{\IEEEauthorblockN{Xi Cheng$^{1,*}$, Ke Liu$^{2,*}$, Siyuan Feng$^{3}$, Jane Lin$^{4}$, H. Oliver Gao$^{1}$}%
\thanks{$^{*}$Equal contribution. Contact: xc557@cornell.edu.}%
\thanks{$^{1}$Cornell University, Ithaca, NY, USA.}%
\thanks{$^{2}$University of California, Berkeley, CA, USA.}%
\thanks{$^{3}$The Hong Kong Polytechnic University, Hong Kong SAR.}%
\thanks{$^{4}$University of Illinois Chicago, Chicago, IL, USA.}}

\maketitle

\begin{abstract}
Foundation models, including large language models (LLMs) and vision-language models (VLMs), are increasingly used for transportation management center (TMC) tasks such as anomaly detection, incident reporting, and traveler information. Deploying multiple such models across TMC functions raises a portfolio question: which model should serve each function, in which deployment mode, and under what shared hardware budget? We formulate this as the Foundation Model Deployment Portfolio (FMDP) problem, a mixed-integer program minimizing total cost of ownership (TCO) subject to per-function quality, latency, and safety constraints over shared GPU capacity. We prove the problem NP-hard by reduction from the 0-1 knapsack problem and propose a polynomial-time greedy
heuristic. In an illustrative case study with five TMC functions and 19 candidate (model, mode) pairs, FMDP identifies a mixed portfolio costing \$34/mo (97\% below the cheapest feasible all-closed-API baseline) by routing four functions to open-source APIs and the one function whose quality floor no open-source model meets to a closed
API. Break-even analysis shows that on-premise GPU investment becomes reasonable only above approximately 309~vision queries/hour or if API prices double.
\end{abstract}


\section{Introduction}
\label{sec:intro}

Transportation management centers (TMCs) operate continuously,
monitoring hundreds of CCTV feeds, detecting and verifying incidents, coordinating response with police and fire agencies, adjusting signal timing and ramp metering, and pushing traveler information to dynamic message signs and 511 traveler information services~\cite{fhwa2024nextgen}. A recent Federal Highway Administration (FHWA) briefing on next-generation Traffic Management Systems recommends that agencies ``assess the balance between on-premises and cloud-enabled capabilities'', weighing AI/ML computing needs against the cost of GPUs~\cite{itsjpo2025nextgen}. To date, we are aware of no formal framework for this assessment.

Many TMC functions already rely on AI models. Classical detectors like YOLO~\cite{ultralytics2023yolov8} handle vehicle counting and basic anomaly detections, but foundation models (FMs) can support broader TMC tasks: VLMs interpret camera scenes for rare-event detection~\cite{zanella2024harnessing}, LLMs predict traffic states from heterogeneous sensor data~\cite{guo2024towards}, draft incident reports~\cite{da2024open}, and advise on signal timing~\cite{lai2025llmlight}. Several
recent surveys document these applications~\cite{wandelt2024large, zhang2026large, fan2024multimodal}.

As capable models become widely available, a practical planning question emerges: \emph{which model should serve each function, and how should it be deployed?} Agencies can now choose closed-source cloud APIs (e.g., GPT-4o at \$2.50/M tokens~\cite{openai2025pricing}, Gemini~Flash at \$0.075/M~\cite{google2025pricing}), open-source models on hosted inference platforms (e.g., Llama-3.1-8B~\cite{grattafiori2024llama} at \$0.18/M via Together.ai~\cite{together2025pricing}), or on-premise (self-hosted) GPU. These options differ substantially in cost, latency, quality, and data sovereignty. Moreover, the decisions are coupled: when multiple functions share a limited on-premise GPU budget, loading a large vision model for detection can consume available memory to force the report generator onto a more expensive cloud API.

Existing work on LLM cost optimization operates at the per-query level. FrugalGPT~\cite{chen2023frugalgpt} cascades models within a single task; RouteLLM~\cite{ongroutellm} and ThriftLLM~\cite{huang2025thriftllm} route individual queries to cheaper alternatives; and INFaaS~\cite{romero2021infaas} automates model-variant selection at runtime. None of these methods address the multi-function setting in which functions compete for shared infrastructure, and none guarantee the operational requirements a TMC imposes. A low-cost model is not useful if it misreads a sensor or produces an incoherent incident report, so deployment must respect a minimum \emph{quality}. It must also meet \emph{latency} targets; FHWA, for instance, names rapid incident detection a primary Traffic Incident Management performance goal~\cite{itsjpo2025nextgen}. And for safety-critical functions such as wrong-way-driver detection, a \emph{safety} bound must hold the false-negative rate (FNR) below an agency-specified threshold. To our knowledge, FM deployment for TMC operations has not been formulated as a portfolio optimization problem under constraints of the quality, latency, safety, and costs. Our contributions are threefold.
\begin{enumerate}[leftmargin=*, nosep]
    \item We formulate the Foundation Model Deployment Portfolio (FMDP) problem as a constrained mixed-integer program that minimizes monthly total cost of ownership (TCO) across functions sharing GPU capacity, subject to quality, latency, and safety constraints.
    \item We prove the FMDP is NP-hard via reduction from the 0-1 knapsack problem and propose a greedy heuristic.
    \item We present an illustrative TMC case study showing 97\% cost reduction over the cheapest feasible all-closed-API baseline and derive per-function API vs.\ on-premise break-even formulas.
\end{enumerate}

\section{Related Work}
\label{sec:related}

\textbf{Foundation models in transportation.}
Several surveys discuss FM applications in intelligent
transportation: Wandelt et al.~\cite{wandelt2024large} review LLMs
across ITS tasks; Zhang et al.~\cite{zhang2026large} focus on LLMs
for mobility forecasting; Fan et al.~\cite{fan2024multimodal} address
multimodal perception and decision-making on complex roads. VLMs have been applied to video anomaly detection~\cite{zanella2024harnessing, yao2022dota}, LLMs to traffic state prediction~\cite{guo2024towards} and signal control~\cite{lai2025llmlight}, and LLM-based agents to incident analysis~\cite{da2024open}. 

\textbf{Cost-efficient model selection and routing.}
FrugalGPT~\cite{chen2023frugalgpt} pioneered LLM cascading, reporting up to 98\% cost reduction by querying models from cheapest to most expensive until a quality bar is met. ThriftLLM~\cite{huang2025thriftllm} formulates ensemble selection with approximation guarantees; RouteLLM~\cite{ongroutellm} trains preference-based routers; Hybrid\,LLM~\cite{dinghybrid} routes between cloud and on-device models; Dekoninck et al.~\cite{dekoninck2025unified} derive provably optimal routing; and FORC~\cite{vsakota2024fly} applies integer programming to per-task assignment. All target a single function and do not model shared GPU capacity.

\textbf{Deployment economics and security.}
Zhang et al.~\cite{zhang2025cloud} analyze cloud vs on-premise deployment as a pricing game between LLM providers, finding that privacy concerns shape equilibrium strategies. Huang et al.~\cite{huang2025middle} propose SOLID, a semi-open on-premise framework securing selected model layers against distillation while preserving fine-tuning flexibility. Both study a single model's deployment choice.

\textbf{Resource allocation and model serving.}
From an operations research point of view, FMDP combines a
multiple-choice assignment structure with knapsack-type GPU-capacity
constraints and a fixed-charge model-sharing structure, which relate FMDP to the Generalized Assignment Problem~\cite{fisher1986multiplier, pentico2007assignment}. INFaaS~\cite{romero2021infaas} automates the selection of models in heterogeneous hardware, and M\'{e}lange~\cite{griggs2024m} formulates the composition of the GPU fleet as a bin packing. These systems operate at per-query or serving-system granularity, whereas FMDP addresses deployment planning across functions.

\section{Problem Formulation}
\label{sec:formulation}
\subsection{Preliminaries}
Consider a TMC operating $F$ functions $\mathcal{F} = \{1, \ldots, F\}$ over a
planning horizon of $H$ hours (e.g., $H = 720$ for one month). Each
function $f$ processes queries at rate $\lambda_f$ (queries/hour) and must
satisfy: (i)~a minimum quality score $q_f^{\min} \in [0,1]$;
(ii)~a maximum per-query latency $L_f$ (seconds); and
(iii)~for safety-critical functions $f \in \mathcal{F}_s \subseteq \mathcal{F}$,
a domain-specific safety constraint $s_{fm} \le S_f^{\max}$.
What $s_{fm}$ measures depends on the task. For video anomaly
detection, a natural choice is FNR~\cite{zanella2024harnessing}; for text generation, it could be a
hallucination rate~\cite{huang2025survey}. The formulation can use any task-specific safety metric that is measured at the model level and constrained by an upper bound.
Table~\ref{tab:notation} summarizes the notations.

\begin{table}[t]
\centering
\caption{\textsc{Table of Notations}}
\label{tab:notation}
\footnotesize
\begin{tabular}{@{}ll@{}}
\toprule
Notation & Description (units) \\
\midrule
\multicolumn{2}{@{}l}{\textit{Sets and indices}} \\
$f \in \mathcal{F}$ & Function index \\
$m \in \mathcal{M}$ & Model index \\
$d \in \mathcal{D}$ & Deployment mode \\
$\mathcal{A}_f \subseteq \mathcal{M} \times \mathcal{D}$ & Feasible (model, mode) pairs for function $f$ \\
\midrule
\multicolumn{2}{@{}l}{\textit{Decision variables}} \\
$x_{fmd} \in \{0,1\}$ & 1 if function $f$ is assigned to $(m,d)$ \\
$y_{md} \in \{0,1\}$ & 1 if model $m$ is loaded in GPU mode $d$ \\
\midrule
\multicolumn{2}{@{}l}{\textit{Operational parameters}} \\
$H$ & Planning horizon (hours) \\
$\lambda_f$ & Query arrival rate (queries/hour) \\
$\tau_f^{\textsf{in}},\; \tau_f^{\textsf{out}}$ & Average input/output tokens per query \\
$q_{fm},\; q_f^{\min}$ & Quality score / minimum required \\
$\ell_{fmd},\; L_f$ & Per-query latency / maximum allowed (s) \\
$s_{fm},\; S_f^{\max}$ & Safety metric / maximum tolerated \\
$\mathrm{mem}_m$ & Model memory footprint (GB) \\
$\mathrm{cap}_d$ & GPU memory capacity per device (GB) \\
$n_{md}$ & GPUs needed: $\lceil \mathrm{mem}_m / \mathrm{cap}_d \rceil$ \\
$G_d$ & Available GPU count in mode $d$ \\
\midrule
\multicolumn{2}{@{}l}{\textit{Cost parameters}} \\
$p_m^{\textsf{in}},\; p_m^{\textsf{out}}$ & API price per token (\$/token) \\
$R_d$ & GPU rental rate (\$/hr per device) \\
$P_d$ & GPU purchase price (\$/device) \\
$T_d$ & Amortization period (months) \\
$\text{TDP}_d$ & Thermal design power (kW per device) \\
$p_e$ & Electricity price (\$/kWh) \\
$\mu_{fmd}$ & GPU utilization factor $\in [0,1]$ \\
\bottomrule
\end{tabular}
\end{table}

Let $\mathcal{M}$ be the set of candidate models and $\mathcal{D}$ the set
of deployment modes. We distinguish four modes by how inference is accessed
and where it runs (Table~\ref{tab:modes}). Pricing varies
enormously~\cite{pan2025cost,ptolemay2025tco,ivanovici2025cost,detectx2025api,sharma2025strategy}:
among closed-source APIs alone, GPT-4o charges \$2.50/M input tokens
while Gemini-2.5-Flash charges \$0.075/M, a 33$\times$ gap.
Open-source hosted APIs are cheaper and allow fine-tuning.
Cloud GPU rental charges by the hour regardless of utilization; the agency
can load any open-weight model using an inference engine (
e.g., vLLM~\cite{kwon2023efficient} or SGLang~\cite{zheng2024sglang}).
On-premise GPUs incur amortized purchase and energy costs but provide full data control.

\begin{table}[t]
\centering
\caption{\textsc{Deployment Modes and Pricing}}
\label{tab:modes}
\footnotesize
\setlength{\tabcolsep}{3pt}
\begin{tabular}{@{}llrl@{}}
\toprule
Mode & Model / Platform & Pricing & Trade-off \\
\midrule
\multirow{4}{*}{Closed API}
 & GPT-4o~\cite{openai2025pricing}             & \$2.50/\$10.00  & Highest quality \\
 & Claude Haiku~4.5~\cite{anthropic2025pricing} & \$1.00/\$5.00   & Mid-tier quality \\
 & GPT-4o-mini~\cite{openai2025pricing}         & \$0.15/\$0.60   & Budget tier \\
 & Gemini-2.5-Flash~\cite{google2025pricing}    & \$0.075/\$0.30  & Lowest closed \\
\midrule
\multirow{2}{*}{Open API}
 & Llama-3.1-8B~\cite{grattafiori2024llama}     & \$0.18/\$0.18   & Fine-tunable \\
 & Qwen2.5-VL-7B~\cite{bai2025qwen25vltechnicalreport}         & \$0.20/\$0.20   & Vision \\
\midrule
\multirow{2}{*}{Cloud GPU}
 & RTX 4090~\cite{runpod2025pricing}             & \$0.35/hr       & Full control \\
 & A100 80GB~\cite{runpod2025pricing}             & \$1.29/hr      & Full control \\
\midrule
\multirow{2}{*}{On-premise}
 & RTX 4090~\cite{nvidia2022rtx4090}             & \$44/mo+energy  & Sovereignty \\
 & A100 80GB~\cite{nvidia2020a100}               & \$417/mo+energy & Sovereignty \\
\bottomrule
\multicolumn{4}{@{}p{0.95\columnwidth}@{}}{\scriptsize
API: input/output per M tokens (Open-source models hosted inference APIs at representative prices).
On-premise: MSRP amortized over 36\,months.
Cloud GPU: on-demand hourly rate.
All prices retrieved Jan.\ 2026; subject to change.}
\end{tabular}
\end{table}

To compute the cost model, we reduce the four modes into two cost families: \emph{per-query variable cost} (both API types) and \emph{fixed infrastructure plus variable energy} (cloud rental and on-premise). We write $\mathcal{D} = \{\textsf{closed api}, \textsf{open api}, \textsf{gpu-rent},
\textsf{gpu-own}\}$. Let $\mathcal{A}_f \subseteq \mathcal{M} \times \mathcal{D}$ denote the feasible candidate set for function~$f$. Each pair $(m,d)$ is
characterized by quality $q_{fm}$, safety metric $s_{fm}$, latency $\ell_{fmd}$ (for $f \in \mathcal{F}_s$), and memory
$\mathrm{mem}_m$. For GPU modes, the required device count is
$n_{md}=\lceil \mathrm{mem}_m/\mathrm{cap}_d\rceil$.

\subsection{Cost Model}
Monthly cost consists of three components: variable API charges, fixed hardware costs, and, for owned hardware, variable energy cost.

\textbf{API cost} ($d = \textsf{(closed/open) api}$). Monthly cost is proportional to
query volume and token length:
\begin{equation}
    C^{\text{api}}_{fm} = \lambda_f \cdot H \cdot
    \bigl(\tau^{\textsf{in}}_f \cdot p^{\textsf{in}}_m
    + \tau^{\textsf{out}}_f \cdot p^{\textsf{out}}_m\bigr)
    \label{eq:api_cost}
\end{equation}
where $\tau^{\textsf{in}}_f, \tau^{\textsf{out}}_f$ are average
input/output tokens per query and
$p^{\textsf{in}}_m, p^{\textsf{out}}_m$ are the model provider's per-token prices.

\textbf{GPU hardware cost (shared).} For
$d \in \{\textsf{gpu-rent}, \textsf{gpu-own}\}$, hosting model~$m$
requires $n_{md}$ devices. Monthly infrastructure cost reads:
\begin{equation}
    C^{\text{hw}}_{md} = \begin{cases}
    R_d \cdot n_{md} \cdot H & d = \textsf{gpu-rent} \\[4pt]
    P_d \cdot n_{md} \,/\, T_d & d = \textsf{gpu-own}
    \end{cases}
    \label{eq:hw_cost}
\end{equation}
Importantly, $C^{\text{hw}}_{md}$ is charged once per model loaded.

\textbf{Energy cost} ($d = \textsf{gpu-own}$ only). Operational energy per function is given by:
\begin{equation}
    C^{\text{energy}}_{fmd} = \text{TDP}_d \cdot n_{md}
    \cdot \mu_{fmd} \cdot H \cdot p_e
    \label{eq:energy_cost}
\end{equation}
where $\mu_{fmd}=\min(\lambda_f \ell_{fmd}/3600,1)$ is the fraction of time the GPU actively processes queries for function~$f$, assuming sequential inference.

\subsection{FMDP Formulation}
\label{sec:formulation_mip}

We formulate FMDP as a mixed-integer program. Binary variable $x_{fmd}$ assigns function~$f$ to model--deployment pair $(m,d)$, and $y_{md}$ indicates whether model~$m$ is loaded on GPU infrastructure~$d$.

\begin{align}
    \min_{x, y} \;\; & \underbrace{
    \sum_{f} \sum_{\substack{(m,d) \in \mathcal{A}_f \\ d = \textsf{(closed/open) api}}}
    \!\!\! C^{\text{api}}_{fm}\, x_{fmd}
    }_{\text{API token cost}}
    \;+\; \underbrace{
    \sum_{m} \sum_{d \neq \textsf{api}}
    \!\!\! C^{\text{hw}}_{md}\, y_{md}
    }_{\text{fixed hardware cost}}  \notag \\
    & \;+\; \underbrace{
    \sum_{f} \sum_{\substack{(m,d) \in \mathcal{A}_f \\
    d = \textsf{gpu-own}}} \!\!\! C^{\text{energy}}_{fmd}\, x_{fmd}
    }_{\text{variable energy cost}}
    \label{eq:obj} \\[4pt]
    \text{s.t.} \;\;
    & \textstyle\sum_{(m,d) \in \mathcal{A}_f} x_{fmd} = 1
    \quad \forall\, f  \label{eq:assign} \\
    & \textstyle\sum_{(m,d) \in \mathcal{A}_f}
      q_{fm} \cdot x_{fmd} \geq q_f^{\min}
    \quad \forall\, f  \label{eq:quality} \\
    & \textstyle\sum_{(m,d) \in \mathcal{A}_f}
      s_{fm} \cdot x_{fmd} \leq S_f^{\max}
    \quad \forall\, f \!\in\! \mathcal{F}_s \label{eq:safety} \\
    & \textstyle\sum_{(m,d) \in \mathcal{A}_f}
      \ell_{fmd} \cdot x_{fmd} \leq L_f
    \quad \forall\, f  \label{eq:latency} \\
    & y_{md} \geq x_{fmd}
    \quad \forall\, f,\; (m,d),\; d \!\neq\! \textsf{api}
    \label{eq:link} \\
    & \textstyle\sum_{m} n_{md} \cdot y_{md} \leq G_d
    \quad \forall\, d \!\neq\! \textsf{api} \label{eq:gpu} \\
    & x_{fmd},\; y_{md} \in \{0,1\} \label{eq:binary}
\end{align}

Constraint~\eqref{eq:assign} assigns exactly one model--mode pair per
function. Constraints~\eqref{eq:quality}--\eqref{eq:latency} enforce
quality floors, safety bounds and
latency ceilings. Since exactly one $x_{fmd}=1$ per function, these reduce to checking whether the selected pair meets each threshold; infeasible pairs can be pre-filtered from~$\mathcal{A}_f$. We keep the
linear-sum form because it produces tighter LP relaxation bounds
(Section~\ref{sec:solution}).
Constraint~\eqref{eq:link} links the two variable families: using
model~$m$ on a GPU for any function forces $y_{md}=1$.
Constraint~\eqref{eq:gpu} ensures that GPU usage in each mode does not exceed $G_d$.

\section{Complexity and Solution Methods}
\label{sec:solution}

\subsection{NP-Hardness}

\begin{proposition}
The FMDP is NP-hard.
\end{proposition}
\begin{proof}
We reduce from the 0--1 Knapsack Problem
(KP)~\cite{karp2009reducibility}, in which items $j = 1, \ldots, J$
with weights $w_j$, values $v_j$, and a capacity $B$ are given.
Construct an FMDP instance with one function $f_j$ per item and
budget $G_{\textsf{gpu-own}} = B$; the mode set $\mathcal{D}$ is
unchanged, and modes other than $\textsf{gpu-own}$ and
$\textsf{open api}$ simply have no feasible candidates. Set the electricity price $p_e = 0$. Each function $f_j$ has two feasible candidates. The first is a distinct local model $m_j$ with
$\mathrm{mem}_{m_j} = w_j \cdot \mathrm{cap}_{\textsf{gpu-own}}$, so
that $n_{m_j,\textsf{gpu-own}} = w_j$ and, by
Eq.~\eqref{eq:hw_cost}, its hardware cost is $\kappa w_j$ with
$\kappa = P_{\textsf{gpu-own}} / T_{\textsf{gpu-own}}$. The second is
a dedicated API model $a_j$ whose per-token price is chosen through
Eq.~\eqref{eq:api_cost} so that its monthly cost equals
$\kappa w_j + v_j$. All quality, latency, and safety constraints are
set non-binding. Because each $m_j$ serves only $f_j$, an assignment
is feasible iff the GPU-assigned subset $S$ satisfies
$\sum_{j \in S} w_j \leq B$, and its total cost is
$\kappa \sum_{j} w_j + \sum_{j \notin S} v_j$. Minimizing cost is
therefore equivalent to maximizing $\sum_{j \in S} v_j$ subject to
the knapsack constraint, so solving FMDP optimally solves KP.
\end{proof}

\subsection{LP Relaxation Bound}
Relaxing the binary variables to $[0,1]$ gives a polynomial-time LP
lower bound. Without the GPU-capacity constraint~\eqref{eq:gpu} and
the linking constraints~\eqref{eq:link}, the problem decomposes into
independent per-function selections whose LP relaxation is integral.
With shared GPU capacity and fixed-charge loading variables,
tightness is no longer guaranteed; we therefore use the LP optimum
as a lower bound to certify heuristic solution quality.

\subsection{Greedy Heuristic}
Algorithm~\ref{alg:greedy} is a two-phase greedy heuristic cheap
enough to re-run whenever prices change. \textbf{Phase~1} scans the
functions in arbitrary order and assigns each to its cheapest
feasible (model, mode) pair given the models already loaded; a model
loaded once serves later functions at no extra hardware cost,
matching the fixed-charge structure of~\eqref{eq:obj}.
\textbf{Phase~2} runs only if the GPU budget~$G_d$ is exceeded. It
first computes the cheapest feasible API
fallback cost for each GPU-assigned function~$\hat{c}_f$, then repeatedly switches the function
with the smallest repair ratio $r_f = \Delta_f / g_f$ to its
fallback, where $\Delta_f = \hat{c}_f - \mathrm{cost}_{fmd}$ is the
resulting cost increase and $g_f$ is the number of GPU devices
released. A function sharing its model with others frees no devices,
so we set $g_f = 0$ and $r_f = +\infty$ rather than divide by zero.
If all $r_f=+\infty$, the heuristic switches the function with the
smallest $\Delta_f$, reducing the number of users of at least one
loaded model and eventually releasing capacity. The loop ends once the budget is met. The Worst-case
complexity is $O(F \cdot |\mathcal{A}|_{\max} + F^2)$: candidate 
lists are scanned once in Phase~1 and once for~$\hat{c}_f$, and
Phase~2 performs repair iterations of at most~$F$, up
to~$F$ functions.

\begin{algorithm}[t]
\caption{Two-Phase Greedy Heuristic for FMDP}
\label{alg:greedy}
\begin{algorithmic}[1]
\STATE \textbf{Input:} functions $\mathcal{F}$, catalogs $\{\mathcal{A}_f\}$, GPU budget $G_d$
\STATE \textbf{Output:} assignment $\sigma$; loaded GPU models $\mathcal{Y} \subseteq \mathcal{M}$, initially $\emptyset$
\STATE \textit{// Phase 1: minimum-cost assignment}
\FOR{each $f \in \mathcal{F}$}
    \STATE $\sigma(f) \leftarrow \arg\min_{(m,d) \in \mathcal{A}_f} \mathrm{cost}_{fmd}(\mathcal{Y})$; \textbf{if} $d = \mathrm{gpu}$ \textbf{then} $\mathcal{Y} \leftarrow \mathcal{Y} \cup \{m\}$
\ENDFOR
\STATE $G_{\text{used}} \leftarrow \sum_{m \in \mathcal{Y}} n_{m,\text{gpu}}$; \textbf{if} $G_{\text{used}} \leq G_d$ \textbf{then return} $\sigma$
\STATE \textit{// Phase 2: capacity repair}
\STATE $\hat{c}_f \leftarrow \min_{(m',d') \in \mathcal{A}_f,\, d' = \mathrm{api}} \mathrm{cost}_{fm'd'}$ for each GPU-assigned $f$
\WHILE{$G_{\text{used}} > G_d$}
    \FOR{each $f$ with $\sigma(f) = (m, \mathrm{gpu})$}
        \STATE $\Delta_f \leftarrow \hat{c}_f - \mathrm{cost}_{fmd}$;\;
          $g_f \leftarrow n_{m,\mathrm{gpu}} \cdot \mathbf{1}[f \text{ sole user of } m]$
        \STATE $r_f \leftarrow \Delta_f / g_f$ \textbf{if} $g_f > 0$ \textbf{else} $+\infty$
    \ENDFOR
    \STATE Switch $f^*$ with smallest $r_f$ (smallest $\Delta_f$ if all $r_f{=}+\infty$); update $\sigma$, $\mathcal{Y}$, $G_{\text{used}}$
\ENDWHILE
\RETURN $\sigma$
\end{algorithmic}
\end{algorithm}

\section{Illustrative Case Study}
\label{sec:case}
We use an illustrative TMC case study to demonstrate the FMDP. Scenario details are demonstrated as follows. In practice, agencies would replace the illustrative entries in Table~\ref{tab:catalog} with measurements from their own models and workloads.

\subsection{Scenario Design}

We consider a TMC with $\mathcal{F}{=}5$ functions and illustrative requirements (Table~\ref{tab:functions}), reflecting the operational scope described in the NextGen TMS briefing~\cite{itsjpo2025nextgen}: video
anomaly detection~(f1), traffic state estimation~(f2), incident report generation~(f3), signal timing advisory~(f4), and traveler information~(f5). Function~f1 is safety-critical: it monitors CCTV feeds for wrong-way drivers, pedestrians on freeways, and road safety, and we assume a FNR bound of $s_{f1} \leq 0.10$.
Query rates range widely, from 200~queries/hour
for f1 (one frame every 18\,s) down to just 3~incidents/hour for f3.

\begin{table}[t]
\centering
\caption{\textsc{Illustrative TMC Functions}}
\label{tab:functions}
\footnotesize
\setlength{\tabcolsep}{3pt}
\begin{tabular}{@{}clrrrl@{}}
\toprule
ID & Function & $\lambda_f$ & $L_f$ & $q_f^{\min}$ & Safety \\
 & & \scriptsize(q/hr) & \scriptsize(s) & & \\
\midrule
f1 & Video anomaly detection\ & 200 & 5 & 0.70 & FNR$\,{\leq}\,$0.10 \\
f2 & Traffic state estimation\ & 120 & 10 & 0.75 & -- \\
f3 & Incident report generation\ & 3 & 30 & 0.80 & -- \\
f4 & Signal timing advisory\ & 60 & 5 & 0.70 & -- \\
f5 & Traveler information & 10 & 15 & 0.75 & -- \\
\bottomrule
\end{tabular}
\end{table}

We assemble a set of candidates per function from the deployment modes in
Table~\ref{tab:modes}; Table~\ref{tab:catalog} lists the full
catalog. Quality and FNR scores are illustrative but align with  published performance tiers: frontier closed-source VLMs outperform open-source VLMs, which in turn outperform classical
detectors~\cite{zanella2024harnessing,yao2022dota,chen2024internvl,ultralytics2023yolov8}; text-function quality (f2--f5) follows the same ordering, consistent with LLM-as-judge evaluations~\cite{zheng2023judging} and public leaderboards (e.g.,~\cite{artificialanalysis2025leaderboard}). We assume that
the relative performance of these models on TMC tasks mirrors their
ranking on public benchmarks.

API costs are computed from Eq.~\eqref{eq:api_cost}. Because providers use different tokenizers, the same prompt results in different token counts and different per-query costs~$c_q$. For vision queries, GPT-4o encodes a 720p CCTV frame as ${\sim}$1{,}105 image tokens~\cite{openai2025pricing}; Qwen2.5-VL tiles the same frame into ${\sim}$1{,}000 tokens~\cite{bai2025qwen25vltechnicalreport}. On-premise costs follow Eqs.~\eqref{eq:hw_cost}--\eqref{eq:energy_cost}. We assume two agency-owned RTX~4090 GPUs (\$1{,}600
MSRP~\cite{nvidia2022rtx4090}, 36-month amortization, 450\,W TDP,
\$0.12/kWh), a workstation-class configuration without data-center infrastructure, for 48\,GB total on-premise VRAM.

\begin{table}[t]
\centering
\caption{\textsc{Model Catalog and Monthly TCO}}
\label{tab:catalog}
\footnotesize
\setlength{\tabcolsep}{3pt}
\begin{tabular}{@{}clllcr@{}}
\toprule
 & Model & Deployment & $q$ & FNR & \$/month \\
\midrule
\multirow{4}{*}{\rotatebox{90}{\scriptsize f1\,(200\,q/h)}}
 & GPT-4o (vision)  & Closed API          & .85 & .04 & 1{,}102 \\
 & Qwen2.5-VL-7B (vision)    & Open API           & .78 & .08 & 29 \\
 & InternVL2-8B (vision)     & On-Premise\,\scriptsize{(16G)} & .78 & .08 & 46 \\
 & YOLOv8-L (vision)         & On-Premise\,\scriptsize{($\!<\!$1G)} & .72 & .11 & $<$1 \\
\midrule
\multirow{4}{*}{\rotatebox{90}{\scriptsize f2\,(120\,q/h)}}
 & GPT-4o-mini       & Closed API          & .82 & --  & 9 \\
 & Gemini-2.5-Flash  & Closed API          & .79 & --  & 3 \\
 & Llama-3.1-8B      & Open API           & .77 & --  & 2 \\
 & Qwen2.5-7B        & On-Premise\,\scriptsize{(15G)} & .77 & --  & 45 \\
\midrule
\multirow{5}{*}{\rotatebox{90}{\scriptsize f3\,(3\,q/h)}}
 & GPT-4o            & Closed API          & .90 & --  & 9 \\
 & Claude Haiku~4.5  & Closed API          & .84 & --  & 4 \\
 & GPT-4o-mini       & Closed API          & .82 & --  & 1 \\
 & Llama-3.1-8B      & Open API           & .78 & --  & $<$1 \\
 & Rule-based        & --                 & .60 & --  & 0 \\
\midrule
\multirow{3}{*}{\rotatebox{90}{\scriptsize f4\,(60\,q/h)}}
 & GPT-4o-mini       & Closed API          & .80 & --  & 6 \\
 & Gemini-2.5-Flash  & Closed API          & .77 & --  & 2 \\
 & Qwen2.5-7B        & Open API           & .75 & --  & 1 \\
\midrule
\multirow{3}{*}{\rotatebox{90}{\scriptsize f5\,(10\,q/h)}}
 & GPT-4o            & Closed API          & .88 & --  & 40 \\
 & GPT-4o-mini       & Closed API          & .80 & --  & 2 \\
 & Llama-3.1-8B      & Open API           & .76 & --  & 1 \\
\bottomrule
\multicolumn{6}{@{}p{\columnwidth}@{}}{\scriptsize
$q$/FNR: Illustrative values informed by published
performance tiers (~\cite{zanella2024harnessing, zheng2023judging}).
API TCO\,$=$\,$\lambda_f \!\cdot\! H \!\cdot\! c_q$ per
Eq.~\eqref{eq:api_cost} with official
pricing~\cite{openai2025pricing,anthropic2025pricing,%
google2025pricing} (Table~\ref{tab:modes});
On-Premise from Eqs.~\eqref{eq:hw_cost}--\eqref{eq:energy_cost}.}
\end{tabular}
\end{table}

\subsection{Results and Analysis}
We implement Algorithm~\ref{alg:greedy} in
Python  on an Apple MacBook Pro (M4, 24\,GB RAM). Experiment 1 presents the FMDP portfolio, while Experiment 2 studies API on-premise break-even points.

\textbf{Experiment~1: Optimal portfolio vs.\ baselines.}
We evaluate five single-tier baselines (Table~\ref{tab:results}),
each forcing every function into one deployment mode:
\emph{All-Closed-Frontier} picks the highest-quality closed model
per function; \emph{All-Closed-Budget} picks the cheapest feasible
closed model; \emph{All-Open-Source-API} uses hosted open-source
models throughout; \emph{All-On-Premise} runs everything locally;
and \emph{Classical-Only} relies on rule-based or classical-ML
methods.

FMDP assigns f1, f2, f4, and f5 to hosted open-source APIs and routes only f3 to GPT-4o-mini, because no open-source option meets its quality floor. The resulting portfolio costs \$34/mo with zero GPU usage. At current open-source API prices (\$0.20/M tokens for vision), f1's 144{,}000 monthly queries cost even less via API than on-premise hardware (\$29 vs.\ \$46/mo).

Three baselines are infeasible. All-On-Premise requires 63\,GB of GPU
memory, exceeding the 48\,GB budget. Classical-Only and
All-Open-Source-API fail the quality threshold for f3 ($q{=}.60$ and
$.78$ vs.\ $q^{\min}{=}.80$). Thus, the all-open-source portfolio nearly matches FMDP in cost but does not satisfy all constraints. FMDP is feasible by assigning f3 to GPT-4o-mini at a marginal cost of \$1/mo, while keeping the other functions on open-source APIs. The two feasible baselines are closed-source and exceed \$1{,}100/mo, driven by f1's sustained vision queries via GPT-4o. FMDP's \$34/mo portfolio is thus, 97\% below the cheapest feasible baseline.

\begin{table}[t]
\centering
\caption{\textsc{Deployment Strategy Comparison}}
\label{tab:results}
\footnotesize
\setlength{\tabcolsep}{3pt}
\begin{tabular}{@{}lrrrc@{}}
\toprule
Strategy & TCO & $\bar{Q}$ & GPU & Feasible? \\
 & \scriptsize(\$/mo) & & \scriptsize(GB) & \\
\midrule
All-Closed-Frontier & 1{,}166 & .85 & 0 & \checkmark \\
All-Closed-Budget   & 1{,}110 & .81 & 0 & \checkmark \\
All-Open-Source-API &     34 & .77 & 0 & $\times$ \\
All-On-Premise      &     226 & .77 & 63 & $\times$ \\
Classical-Only      &   $<$1  & .63 & 0 & $\times$ \\
\textbf{FMDP Optimal} & \textbf{34} & \textbf{.78} & \textbf{0} & \checkmark \\
\bottomrule
\multicolumn{5}{@{}p{\columnwidth}@{}}{\scriptsize
Frontier: highest-$q$ closed model per function.
Budget: cheapest feasible closed model.
Open-Source: Qwen2.5-VL~(vision), Llama~(text);
fails $q^{\min}$ on f3.
On-Premise: needs 63\,GB $>$ 48\,GB budget.
Classical: fails $q^{\min}$ on f3.
$\bar{Q} {=} \tfrac{1}{|\mathcal{F}|}\sum_f q_{fm}$.}
\end{tabular}
\end{table}

\textbf{Experiment~2: API vs.\ on-premise break-even points.}
On-premise deployment becomes cheaper than a cloud API once the
query rate exceeds
\begin{equation}
    \lambda^* \;=\; \frac{C^{\text{hw}}_{md}}{H \cdot c_q}
    \label{eq:tipping}
\end{equation}
where $c_q$ is the per-query API cost and $C^{\text{hw}}_{md}$ is
the monthly hardware amortization. Table~\ref{tab:tipping} evaluates
this for a single RTX~4090 at \$44/mo.

At \$0.20/M tokens, the open-source vision API (Qwen2.5-VL) breaks
even at 309~q/h, above f1's 200~q/h, so FMDP keeps f1 on the API.
Adding cameras or raising the frame rate beyond one frame every
12\,s would cross this threshold and shift f1 on-premise. GPT-4o,
 on the contrary, breaks even at only 8~q/h due to its higher
per-query cost (\$0.0077): for any sustained vision workload beyond
that rate, owning a GPU beats paying closed-API vision prices. Text
workloads are far below their break-even points, 617~q/h for
GPT-4o-mini and 2{,}684~q/h for Llama-3.1-8B. Therefore, no text function
in this scenario justifies local hardware. These thresholds move
with API pricing, which changes frequently: doubling the vision
price to \$0.40/M halves $\lambda^*$ to 154~q/h, below f1's rate,
and the optimal assignment shifts on-premise.

\begin{table}[t]
\centering
\caption{\textsc{API vs.\ On-Premise Break-Even Points}}
\label{tab:tipping}
\footnotesize
\setlength{\tabcolsep}{3pt}
\begin{tabular}{@{}llrl@{}}
\toprule
Cloud API & $c_q$ (\$/q) & $\lambda^*$ (q/h) & Implication \\
\midrule
GPT-4o (vision)      & 0.0077    &     8 & on-prem  \\
Qwen2.5-VL (vision)  & 0.00020   &   308.6 & use API  \\
GPT-4o-mini (text)   & 0.00010   &   617.3 & use API  \\
Llama-3.1-8B (text)  & 0.000023  & 2{,}683.8 & use API  \\
\bottomrule
\multicolumn{4}{@{}p{\columnwidth}@{}}{\scriptsize
$\lambda^*$: query rate above which on-premise is cheaper,
per Eq.~\eqref{eq:tipping} with
$C^{\text{hw}} {=} \$44$/mo (one RTX~4090).
``On-prem'': $\lambda^* < \lambda_{f1}{=}200$.
``Use API'': $\lambda^* > \max_f \lambda_f$.}
\end{tabular}
\end{table}

\section{Discussion}
\label{sec:discussion}
Although the case study is small, it illustrates why FM deployment for TMCs should be treated as a portfolio problem rather than a set of independent model choices. No single deployment tier can be both feasible and cost-minimal. The all-open-source API strategy nearly matches FMDP’s monthly cost but fails the quality requirement for incident report generation, while the cheapest feasible all-closed-source strategy exceeds \$1{,}100/mo because sustained vision API calls dominate cost. FMDP identifies the narrow point where closed-source capability is needed and keeps the remaining functions on lower-cost hosted open-source APIs. Break-even analysis shows that deployment decisions are sensitive to query volume and per-query price. At January 2026 prices, the open-source vision API remains cheaper than local deployment at 200~q/h, but this would reverse if camera sampling rates, the number of streams, or API prices increase. Text functions have much higher break-even rates, thus on-premise GPU deployment is not cost effective for the small volume reporting and advisory workloads in our case study.

\section{Conclusion}
\label{sec:conclusion}
This paper introduced FMDP, a mixed-integer programming framework to jointly select foundation models and deployment modes across TMC functions subject to the constraints of quality, latency, safety, and shared GPU-capacity. We proved FMDP is NP-hard and proposed a two-phase greedy heuristic. In an illustrative five-function case study, FMDP achieves a \$34/mo portfolio, 97\% below the cheapest feasible all-closed-source baseline, by using hosted open-source APIs where possible and a closed-source API only where needed. The break-even analysis further identifies when on-premise GPU deployment becomes cost-effective. 

\textbf{Limitations and future work.}
The case study demonstrates the planning methodology rather than
realistic operational performance: its quality, latency, and safety
parameters are illustrative and should be replaced with
agency-specific measurements prior to operational use. The current formulation assigns one model to each function and does not consider model cascading, price and demand uncertainty, or inter-function dependencies. Future work will extend FMDP to model cascading, multi-period planning, and validation with operational TMC data.

\bibliographystyle{ieeetr}
\bibliography{reference}

\end{document}